\documentclass{article}

\usepackage{graphicx} 
\usepackage{subfigure} 

\usepackage{natbib}

\usepackage{algorithm}
\usepackage{algorithmic}

\usepackage{hyperref}

\usepackage[accepted]{icml2014}

\icmltitlerunning{Discriminative Features via Generalized Eigenvectors}

\usepackage{amsmath}
\usepackage{amsthm}
\usepackage{amssymb}
\usepackage[normalem]{ulem}
\usepackage{comment}
\usepackage{multirow}

\newtheorem{prop}{Proposition}

\renewcommand{\P}{\mathbb{P}}
\newcommand{\R}{\mathbb{R}}

\newcommand{\E}{\mathbb{E}}
\newcommand{\I}{\mathbb{I}}

\newcommand{\D}{\mathcal{D}}

\DeclareMathOperator*{\Cov}{Cov}
\DeclareMathOperator*{\Tr}{Trace}

\begin{document} 

\twocolumn[
\icmltitle{Discriminative Features via Generalized Eigenvectors}

\icmlauthor{Nikos Karampatziakis}{nikosk@microsoft.com}
\icmlauthor{Paul Mineiro}{pmineiro@microsoft.com}
\icmladdress{Microsoft CISL, 1 Microsoft Way, Redmond, WA 98052 USA}

\icmlkeywords{feature learning, supervised learning, multiclass classification}

\vskip 0.3in
]

\begin{abstract} 
    Representing examples in a way that is compatible with the 
    underlying classifier can greatly enhance the performance of 
    a learning system. In this paper we investigate scalable 
    techniques for inducing discriminative features by taking 
    advantage of simple second order structure in the data. We
    focus on multiclass classification and show that features
    extracted from the generalized eigenvectors of the class
    conditional second moments lead to classifiers with excellent 
    empirical performance. Moreover, these features have attractive
    theoretical properties, such as inducing representations that are
    invariant to linear transformations of the input. We evaluate
    classifiers built from these features on three different tasks,
    obtaining state of the art results.
\end{abstract}

\section{Introduction}
\label{intro}
 
Supervised learning has been a great success story for machine
learning, both in theory and in practice.  In theory, we have
a good understanding of the conditions under which supervised
learning can succeed \cite{vapnik1998statistical}. In practice,
supervised learning approaches are profitably employed in many
domains, from movie recommendation to speech and image recognition
\cite{koren2009matrix,hinton2012deep,krizhevsky2012imagenet}. The
success of all of these systems crucially hinges on the compatibility
between the model and the representation used to solve the problem.

For some problems, the kinds of representations and models that
lead to good performance are well-known.  In text classification,
for example, unigram and bigram features together with linear
classifiers are known to work well for a variety of related
tasks \cite{halevy2009unreasonable}.  For other problems,
such as drug design, speech, and image recognition, far less
is known about which combinations are effective.  This has
fueled interest in methods that can learn the appropriate
representations directly from the raw signal, with techniques
such as dictionary learning~\cite{mairal2008supervised} and deep
learning~\cite{krizhevsky2012imagenet,hinton2012deep} achieving
state of the art performance in many important problems.

In this work, we explore conceptually and computationally simple
ways to create discriminative features that can scale to a large
number of examples, even when data is distributed across many
machines. Our techniques are not a panacea. They are exploiting
simple second order structure in the data and it is very easy to
come up with sufficient conditions under which they will not give
any advantage over learning using the raw signal.  Nevertheless,
they empirically work remarkably well.

Our setup is the usual multiclass setting where we are given labeled
data $\{x_i,y_i\}_{i=1}^n$, sampled iid from a distribution $\D$
on $\R^d \times [k]$, and we need to come up with a classifier $h:
\R^d \to [k]$ with low generalization error $\P_\D(h(x)\neq y)$.
Abusing notation, we will sometimes use $y$ to refer to the one
hot encoding of $y$ that identifies each class with one of the
vertices of the standard $k-1$-simplex.  To keep the focus on the
quality of our feature representation we will restrict ourselves to
$h$ being linear, such as a multiclass linear SVM or multinomial
logistic regression.  We suspect representations that improve the
performance of linear classifiers will also beneficially compose
with nonlinear techniques.

\section{Method}

One of the simplest possible statistics involving both features and labels is the matrix $\E[x y^\top]$,
which in multiclass classification is the collection of class-conditional mean feature vectors.
This statistic has been thoroughly explored, e.g., Fisher LDA~\cite{fisher1936use} 
and Sliced Inverse Regression~\cite{li1991sliced}.  However, in many practical applications we
expect that the data distribution contains much more information
than that contained in the first moment statistics.  The natural next object of study
is the tensor $\E[x \otimes x \otimes y]$.

In multiclass classification, the tensor $\E[x\otimes x \otimes
y]$ is simply a collection of the conditional second moment
matrices $C_i=\E[xx^\top | y=i]$.  There are many standard ways of
extracting features from these matrices. For example, one could try
per-class PCA \cite{wold1977simca} which will find directions that
maximize $v^\top\E[xx^\top | y=i]v=\E[(v^\top x)^2 | y=i]$, or VCA
\cite{livni2013vanishing} which will find directions that minimize
the same quantity. The subtlety here is that there is no reason to
believe that these directions are specific to class $i$. In other
words, the directions we find might be very similar for
all classes and, therefore, not be discriminative.

A simple alternative is to work with the quotient 
\begin{equation}
\label{eqn:rayleigh} 
R_{ij}(v)=\frac{\E[(v^\top x)^2 | y=i]}{\E[(v^\top x)^2 | y=j]}=
\frac{v^\top C_i v}{v^\top C_j v},
\end{equation} 
whose local maximizers are the generalized eigenvectors solving
$C_i v = \lambda C_j v$. Despite the non-convexity, efficient and
robust routines for solving these types of problems are part of
mature software packages such as LAPACK.

Since objective~\eqref{eqn:rayleigh} is homogeneous in $v$, we will
assume that each eigenvector $v$ is scaled such that $v^\top C_j
v=1$. Then we have that $v^\top C_i v=\lambda$, i.e.~on average,
the squared projection of an example from class $i$ on $v$ will be
$\lambda$ while the squared projection of an example from class $j$
will be $1$. As long as $\lambda$ is far from 1, this gives us a
direction along which we expect to be able to discriminate the two
classes by simply using the magnitude of the projection. Moreover,
if there are many eigenvalues substantially different from 1 all
associated eigenvectors can be used as feature detectors.

\subsection{Useful Properties}

The feature detectors resulting from maximizing equation \eqref{eqn:rayleigh}
have two useful properties which we list below.  For simplicity
we state the results assuming full rank exact conditional moment
matrices, and then discuss the impact of regularization and finite
samples.

\begin{prop} \label{prop:invariance}
(Invariance) Under the above assumptions,
the embedding $v^\top x$ is invariant to invertible linear
transformations of $x$.
\end{prop}  
\begin{proof}
Let  $A \in \R^{d \times d}$ be invertible and $x'=Ax$ be the
transformed input.  Let $C_m=\E[xx^\top|y=m]$ be the second moment
matrix given $y=m$ for the original data. For a class pair $(i,j)$,
a generalized eigenvector $v$ satisfies $C_i v= \lambda C_j v$.
Using the Cholesky factorization $C_j=L_j L_j^\top$ and setting
$v=L_j^{-\top}u$ we have
\begin{equation} \label{eq:gevreduction}
L_j^{-1}C_i L_j^{-\top} u=\lambda u,
\end{equation} i.e.,\ $u$ is an eigenvector of $L_j^{-1}C_i L_j^{-\top}$. 
Moreover, the embedding for the original data involves only 
\begin{equation} \label{eq:embed}
v^\top x=u^\top L_j^{-1} x.
\end{equation} 
For the transformed data, the conditional second moments are
$\E[x'x'^\top|y=m]=A\E[xx^\top|y=m]A^\top=AC_m A^\top$ and the
corresponding generalized eigenvector $v'$ satisfies 
$AC_i A^\top v'= \lambda AC_j A^\top v'$.  Letting 
$v'=A^{-\top} L_j^{-\top} u'$ we see that $u'$ satisfies
$L_j^{-1}C_i L_j^{-\top} u'=\lambda u'$ which is the same as
\eqref{eq:gevreduction}.  Therefore $u'$ can be chosen such that
$u'=u$. Finally, the embedding involves only 
$v'^\top x' = u'^\top L_j^{-1}A^{-1} A x = u^\top L_j^{-1} x$ which
is the same as the embedding \eqref{eq:embed} for the original data.
\end{proof}

It is worth pointing out that the results of some popular methods,
such as PCA, are not invariant to linear transformations of
the inputs. For such methods, differences in preprocessing and
normalization can lead to vastly different results.  The practical
utility of an ``off the shelf'' classifier is greatly improved by
this invariance, which provides robustness to data specification,
e.g., differing units of measurement across the original features.

\begin{prop}
(Diversity) Two feature detectors $v_1$ and $v_2$ extracted from the 
same ordered class pair $(i,j)$ have uncorrelated responses
$\E[(v_1^\top x)(v_2^\top x)|y=j]=0.$
\end{prop}  
\begin{proof}
This follows from the orthogonality of the eigenvectors in the
induced problem $L_j^{-1}C_i L_j^{-\top} u=\lambda u$ (c.f. proof
of Proposition~\ref{prop:invariance}) and the connection  
$v= L_j^{-\top}u$.  If $u_1$ and $u_2$ are eigenvectors of
$L_j^{-1}C_i L_j^{-\top}$ then 
$ 0 = u_1^\top u_2 = v_1^\top L_j L_j^\top v_2 = v_1^\top \E[xx^\top|y=j] v_2 = \E[(v_1^\top x)(v_2^\top x)|y=j].  $
\end{proof}

Diversity indicates the different generalized eigenvectors per class pair 
provide complementary information, and that techniques which only use
the first generalized eigenvector are not maximally exploiting the data.

\subsection{Finite Sample Considerations}

Even though we have shown the properties of our method assuming
knowledge of the expectations $\E[xx^\top|y=m]$, in practice we
estimate these quantities from our training samples. The empirical
average
\begin{equation}\label{eq:empirical}
\hat C_m = \frac{\sum_{i=1}^n \I[y_i=m] x_i x_i^\top}{\sum_{i=1}^n \I[y_i=m]}
\end{equation}
converges to the expectation at a rate of $O(n^{-1/2})$.
Here and below we are suppressing the dependence upon 
the dimensionality $d$, which we consider fixed. Typical
finite sample tail bounds become meaningful once
$n=O(d\log d)$~\cite{vershynin2010introduction}.

Given $\hat C_m= C_m +E_m$ with $||E_m||_2=O(n^{-1/2})$, we can
use results from matrix perturbation theory to establish that our
finite sample results cannot be too far from those obtained using
the expected values.  For example, if the \emph{Crawford number}
\[
c(C_i,C_j) \doteq \min_{||v||=1} (v^\top C_i v)^2 + (v^\top C_j v)^2 > 0,
\]
and the perturbations $E_i$ and $E_j$ satisfy
\[
||E_i||_2^2+||E_j||_2^2 < c(C_i,C_j),
\]
then \cite{golub2012matrix} for all $q \in [d]$
\[
\tan(|\tan^{-1}(\lambda_q) - \tan^{-1}(\hat \lambda_q)|) \leq
 O\left(\frac{1}{\sqrt n c(C_i,C_j)}\right),
\]
where $\lambda_q,\hat \lambda_q$ are the $q$-th generalized
eigenvalues of the matrix pairs $C_i,C_j$ and $\hat C_i,\hat C_j$
respectively. Similar results apply to the sine of the angle 
between an estimated generalized eigenvector and the true one
\cite{demmel2000templates} Section~5.7. 

\subsection{Regularization}

An additional concern with finite samples is that $\hat C_m$ may
not be full rank as we have assumed until now. In particular, if
there are fewer than $d$ examples in class $m$, then $\hat C_m$
is guaranteed to be rank deficient.  When such a matrix appears
in the denominator of \eqref{eqn:rayleigh}, estimation of the
eigenvectors can be unstable and overly sensitive to the sample at
hand. A common solution \cite{platt2010translingual} is to regularize
the denominator matrix by adding a multiple of the identity to the
denominator, i.e., maximizing
\begin{equation}
R^\gamma_{ij}(v)=\frac{v^\top \hat  C_i v}{v^\top ( \hat C_j +\gamma I)v},
\label{eqn:regsignoise}
\end{equation}
which is equivalent to maximizing equation \eqref{eqn:rayleigh}
with an additional upper-bound constraint on the norm of $v$.
We typically set $\gamma$ to be a small multiple of the average
eigenvalue of $\hat C_j$ \cite{friedman1989regularized} which
can be easily obtained as the trace of $\hat C_j$ divided by $d$.
In Section~\ref{sec:experiments} we find this strategy empirically
effective.

\subsection{An Algorithm}

\begin{algorithm}[t]
\begin{algorithmic}[1]
\REQUIRE $S=\{(x_i,y_i)\}_{i=1}^n$, $\theta \geq 0$ and $\gamma \geq 0$
\STATE $F \leftarrow \emptyset$
\FOR{$(i,j \neq i) \in \{ 1, \ldots, k \}^2$}
\STATE Solve $\hat C_i V = (\hat C_j+\frac \gamma d \Tr(\hat C_j) I) V \Lambda $
\STATE $F \leftarrow F \cup \{ V_{q} | \Lambda_{qq} \geq \theta \}$
\ENDFOR
\STATE $\psi_{v,\alpha,\delta}(x) \doteq \max(0,\delta v^\top x)^{\alpha/2}$\\
\STATE $\phi(x) \doteq [\psi_{v,\alpha,\delta}(x) | v,\alpha,\delta \in F \times \{1,2,3\} \times \{-1,1\}]$ \\
\STATE $w = \mathrm{MultiLogit} (\{ (\phi(x), y) | (x, y) \in S \})$
\end{algorithmic}
\caption{Generalized Eigenvectors for Multiclass}
\label{alg:gem}
\end{algorithm}

We are left with specifying a full algorithm for multiclass
classification.  First we need to specify how to use the eigenvectors
$\{ v_i \}$. The eigenvectors define an embedding for each
example $x$ using the projection magnitudes $\{ v_i^\top x \}$
as new coordinates. However the embedding is linear, therefore
composition with a linear classifier is equivalent to learning a
linear classifier in the original space, perhaps with a different
regularization.  This motivates the use of nonlinear functions of
the projection magnitude.

To construct nonlinear maps, we can get inspiration from the
optimization criterion in equation \eqref{eqn:rayleigh}, i.e., the
ratio of expected projection magnitudes conditional on different
class labels.  For example, we could use a nonlinear map such as
$(v^\top x)^2$. This type of nonlinearity can be sensitive (for
example, it is not Lipschitz) so in practice more robust proxies
can be used such as $|v^\top x|$ or even $|v^\top x|^{1/2}$.\footnote{
These choices are simple and yield only slightly worse results 
than what we report in our experiments.} In
principle, smoothing splines or any other flexible set of univariate
basis functions could be used.  In our experiments we simply fit a
piecewise cubic polynomial on $|v^\top x|^{1/2}$. The polynomial has
only two pieces, one for $v^\top x>0$ and one for $v^\top x \leq 0$.
We briefly experimented with interaction terms between projection
magnitudes, but did not find them beneficial.
 
Additionally, we need to address from which class pairs to extract
eigenvectors.  A simple and empirically effective approach, suitable
when the number of classes is modest, is to just use all ordered
pairs of classes. This can be wasteful if two classes are never
confused.  The alternative, however, of leaving out a pair $(i,j)$
is that the classifier might have no way of distinguishing between
these two classes.  Since we do not know upfront which pairs of
classes will be confused, our brute force approach is just a safe
way to endow the classifier with enough flexibility to deal with
any pair of classes that could potentially be confused.  Of course,
as the number of classes grows, this brute force approach becomes
less viable both computationally (due to the quadratic increase
in generalized eigenvalue problems) and statistically (due to
the increase in the number of features for the final classifier).
We discuss issues regarding large numbers of classes in 
Section~\ref{sec:discussion}.

Finally, the generalized eigenvalues can guide us in picking a subset
of the $d$ generalized eigenvectors we could extract from each class
pair, i.e., generalized eigenvalues are useful for feature selection.
A generalized eigenvector $v$ with eigenvalue $\lambda$ has 
$\E[(v^\top x)^2 | y]$ equal to $1$ for the denominator class $y=j$ and
equal to $\lambda$ for the numerator class $y=i$. Therefore, eigenvalues far
from 1 correspond to highly discriminative features. Similar to
\cite{platt2010translingual}, we extract the top few eigenvectors, 
as top eigenspaces are cheaper to compute than bottom eigenspaces.
To guard against picking non-discriminative eigenvectors, we 
discard those whose eigenvalues are less than a threshold $\theta>1$.

The above observations lead to the GEM procedure outlined in
Algorithm~\ref{alg:gem}.  Although Algorithm~\ref{alg:gem} has proven sufficiently versatile
for the experiments described herein, it is merely an example of
how to use generalized eigenvalue based features for multiclass
classification.  Other classification techniques could benefit from
using the raw projection values without any nonlinear manipulation,
e.g., decision trees; additionally the generalized eigenvectors
could be used to initialize a neural network architecture as a form
of pre-training.

We remark that each step in Algorithm~\ref{alg:gem} is highly
amenable to distributed implementation: empirical class-conditional
second moment matrices can be computed using map-reduce techniques,
the generalized eigenvalue problems can be solved independently in
parallel, and the logistic regression optimization is convex and
therefore highly scalable~\cite{agarwal2011}.

\section{Related Work}\label{sec:related}

\begin{table}
\begin{center}
\begin{tabular}{|c|c|c|}
\hline
Method & Signal & Noise \\
\hline\hline
PCA          & $\E[xx^\top]$                & $I$\\
VCA          & $I$                          & $\E[xx^\top]$\\
Fisher LDA   & $\E_y[\E[x|y]\E[x|y]^\top]$  & $\sum_y \Cov[x|y]$\\
SIR          & $\sum_y \E[w|y]\E[w|y]^\top$ & $I$\\
Oriented PCA & $\E[xx^\top]$                & $\E[zz^\top]$\\
Our method &  $\E[xx^\top|y=i]$             & $\E[xx^\top|y=j]$\\
\hline
\end{tabular}
\caption{Table of related methods
(assuming $\E[x]=0$) for finding directions that maximize the
signal to noise ratio. $\Cov[x|y]$ refers to the conditional
covariance matrix of $x$ given $y$, $w$ is a whitened version
of $x$, and $z$ is any type of noise meaningful to the task at
hand. }
\label{tab:signalnoise}
\end{center}
\vspace{-20pt}
\end{table}

Our approach resembles many existing methods that work by finding
eigenvectors of matrices constructed from data. One can think
of all these approaches as procedures for finding directions $v$
that maximize the signal to noise ratio
\begin{equation}
R(v)=\frac{v^\top S v}{v^\top N v},
\label{eqn:signoise}
\end{equation}
where the symmetric matrices $S$ and $N$ are such that the
quadratic forms $v^\top S v$ and $v^\top N v$ represent the signal
and the noise, respectively, captured along direction $v$. In
Table~\ref{tab:signalnoise} we present many well known approaches
that could be cast in this framework.  Principal Component Analysis
(PCA) finds the directions of maximal variance without any particular
noise model.  The recently proposed Vanishing Component Analysis
(VCA) \cite{livni2013vanishing} finds the directions on which the
projections vanish so it can be thought as swapping the roles of
signal and noise in PCA. Fisher LDA maximizes the variability
in the class means while minimizing the within class variance.
Sliced Inverse Regression first whitens $x$, and
then uses the second moment matrix of the conditional whitened means
as the signal and, like PCA, has no particular noise model.  Finally,
oriented PCA \cite{diamantaras1996principal,platt2010translingual}
is a very general framework in which the noise matrix can be the
correlation matrix of any type of noise $z$ meaningful to the task
at hand.

By closely examining the signal and noise matrices, it is clear that
each method can be further distinguished according to two other
capabilities: whether it is possible to extract many directions,
and whether the directions are discriminative. For example, PCA and
VCA can extract many directions but these are not discriminative.
In contrast,  Fisher LDA and SIR are discriminative but they work
with rank-$k$ matrices so the number of directions that could be
extracted is limited by the number of classes. Furthermore both of
these methods lose valuable fidelity about the data by using the
conditional means.

Oriented PCA is sufficiently general to encompass our technique as a
special case. Nonetheless, to the best of our knowledge, the specific
signal and noise models in this paper are novel and, as we show in
Section~\ref{sec:experiments}, they empirically work very well.

\section{Experiments}\label{sec:experiments}

\subsection{MNIST}

\begin{figure}
\includegraphics[trim=30 60 15 15 clip,width=\linewidth]{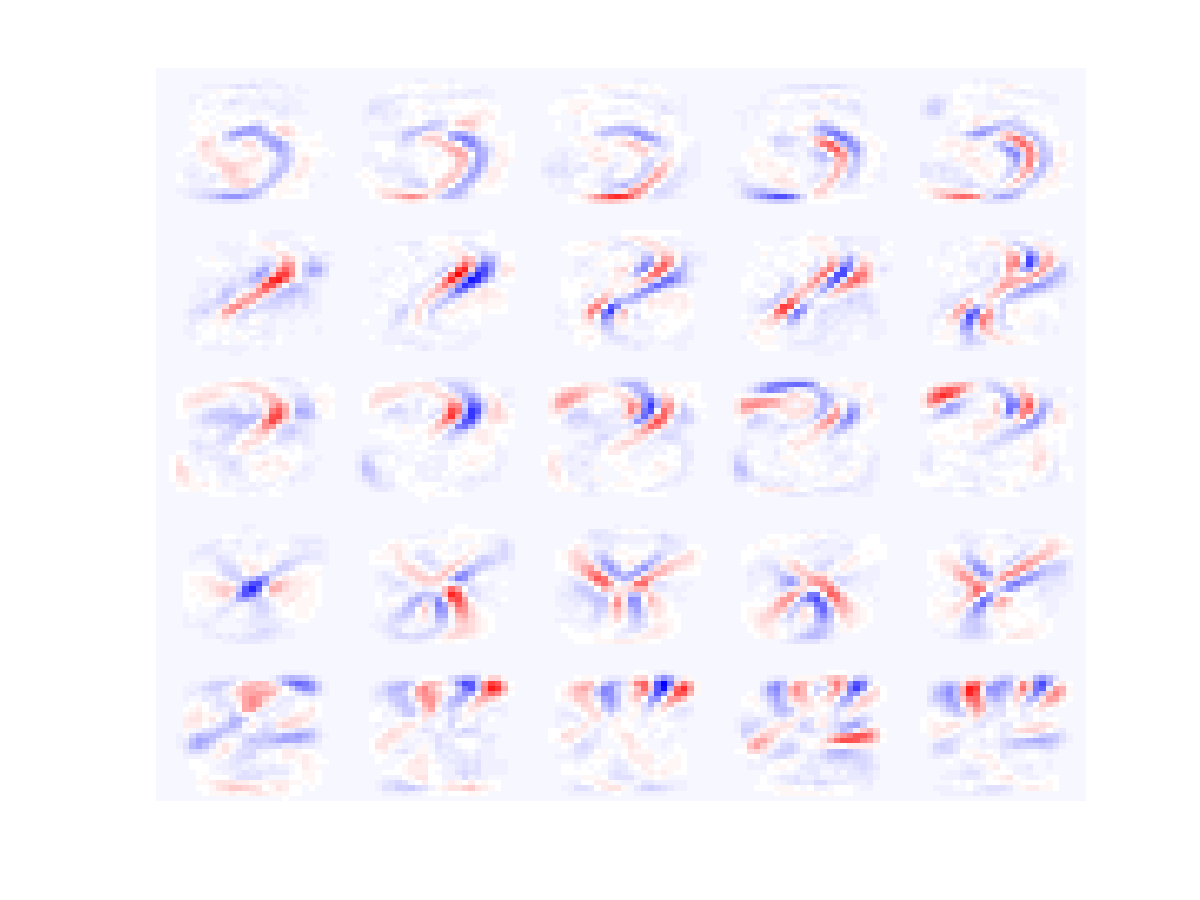} 
\caption{Pictures of the top 5 generalized eigenvectors for MNIST for class pairs $(3, 2)$
    (top row), $(8, 5)$ (second row), $(3, 5)$ (third row), $(8, 0)$ (fourth
    row), and $(4, 9)$ (bottom row) with $\gamma = 0.5$.  Filters have large response on the first 
class and small response on the second class. Best viewed in color.}
\label{fig:mnistfeats}
\vspace{-12pt}
\end{figure}

We begin with the MNIST database of handwritten
digits~\cite{mnistlecun}, for which we can visualize the generalized
eigenvectors, providing intuition regarding the discriminative
nature of the computed directions.  For each of the ten classes,
we estimated $C_m = \E[xx^\top|y=m]$ using \eqref{eq:empirical}
and then extracted generalized eigenvectors for each class
pair $(i, j)$ by solving $\hat C_i v = \lambda (\frac{\gamma}{d} \Tr(\hat C_j) I + \hat C_j) v$.
Figure~\ref{fig:mnistfeats} shows a sample of results from this procedure
for five class pairs (one in each row) and $\gamma=0.5$. In the top row
we use class pair $(3, 2)$ and we observe that the eigenvectors 
are sensitive to the circular stroke of a typical 3 while remaining 
insensitive to the areas where 2s and 3s overlap. Similar results 
are seen in the second and third rows where we use class pairs $(8, 5)$
and $(3, 5)$: the strokes we find are along areas used by the first class
and mostly avoided by the second class. In the fourth row we use class 
pair $(8, 0)$. Here we observe two patterns. First, a dot in the center that 
avoids the 0s. The other 4 detectors consist of positive (red) and negative (blue) strokes
arranged in a way that would cancel each other if we take the inner product 
of the detector with a radially symmetric pattern such as a 0.  Similarly in the bottom
row with class pair $(4, 9)$, the detector attempts to cancel the horizontal stroke corresponding
to the top of the 9, where a typical 4 would be open.

\begin{figure}
\begin{center}
\includegraphics[trim=140 270 135 260, clip, width=0.9 \linewidth]{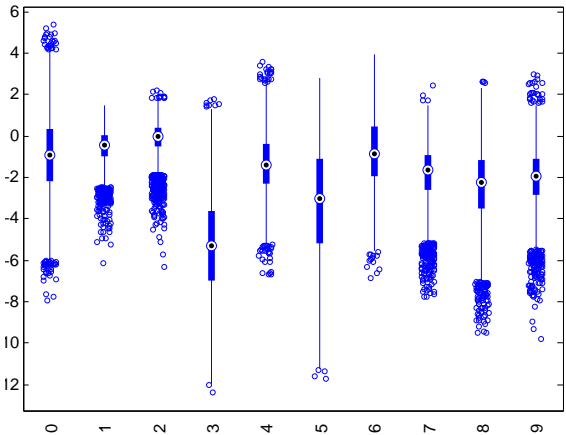}
\caption{Boxplot of the projection onto the first generalized
eigenvector for class pair $(3, 2)$ across the MNIST training set grouped by
label.  Squared projection magnitude on 2s is on average unity, whereas on 3s it
is the eigenvalue. Large responses can appear in other classes (e.g.,
5s and 8s), but this is not guaranteed by construction.}
\label{fig:mnistproj}
\end{center}
\vspace{-20pt}
\end{figure}

Figure~\ref{fig:mnistproj} shows for each of the ten classes the
distribution of values obtained by projecting the training examples
in that class onto the first eigenvector for class pair 
$(3, 2)$, i.e., the top left image in Figure~\ref{fig:mnistfeats}.
The projection pattern inspires two comments.  First, while
the magnitude of the projection is itself discriminative for
distinguishing between 2s and 3s, there is additional information in
knowing the sign of the projection.  This motivates our particular
choice of nonlinear expansion in Algorithm~\ref{alg:gem}.  Second,
the detector is discriminative for class 3 vs. class 2 as per design,
but also useful for distinguishing other classes from 2s.  However
certain classes such as 1s and 7s would be completely confused with
2s were this the only feature.  The number of classes in MNIST is
modest ($k=10$) so we can easily afford to extract features for all
$k(k-1)$ class pairs for excellent discrimination.  For problems with
a large number of classes, however, we need to carefully pick the
subproblems we need to solve so that the resulting set of features
is discriminative, diverse, and complete.  We revisit this topic
in Section~\ref{sec:discussion}.

Table~\ref{tab:mnistres} contains results for algorithm~\ref{alg:gem} 
on the MNIST test set. To determine the hyperparameter settings $\gamma$ and $\theta$, we
held out a fraction of the training set for validation. Once $\gamma$
and $\theta$ were determined, we trained on the entire training set.

For ``deep GEM'' we applied GEM to the representation created by GEM,
i.e., line 7 of Algorithm~\ref{alg:gem}.  Because of the intermediate
nonlinearity this is not equivalent to a single application of GEM,
and we do observe an improvement in generalization.  Subsequent
recursive compositions of GEM degrade generalization, e.g., 3 levels
of GEM yields 110 test errors.  We would like to better understand
the conditions under which composing GEM with itself is beneficial.

Our results occupy an intermediate position amongst state of the
art results on MNIST. For comparison we include results from other permutation-invariant methods from \cite{wan2013regularization} and \cite{goodfellow2013maxout}. These methods rely on generic non-convex 
optimization techniques and face challenging scaling issues 
in a distributed setting~\cite{NIPS2012_0598}.
While maximization of the Rayleigh quotient~\eqref{eqn:rayleigh}
is non-convex, mature implementations are computationally efficient
and numerically robust.  The final classifier is built using convex
techniques and our pipeline is particularly well suited to the
distributed setting, as discussed in Section~\ref{sec:discussion}.

\begin{table}
\begin{center}
\begin{tabular}{c|c}
Method & Test Errors \\ \hline
Dropout     & 120 \\
DropConnect & 112 \\
GEM & 108  \\
deep GEM & 96 \\
Maxout & 94 \\
\end{tabular}
\caption{Test errors on MNIST.  All techniques are permutation invariant and do not augment the training set.}
\label{tab:mnistres}
\end{center}
\vspace{-12pt}
\end{table}

\subsection{Covertype}

Covertype is a multiclass data set whose task is to
predict one of 7 forest cover types using 54 cartographic
variables~\cite{blackard1999comparative}.  RBF kernels provide
state of the art performance on Covertype, and consequently
it has been a benchmark dataset for fast approximate kernel
techniques~\cite{rahimi2007random,manik}.  Here, we demonstrate that
generalized eigenvector extraction composes well with randomized
feature maps in the primal.  This approximates generalized
eigenfunction extraction in the RKHS, while retaining the speed
and compactness of primal approaches.

Covertype does not come with a designated test set, so we randomly
permuted the data set and used the last 10\% for testing, utilizing
the same train-test split for all experiments.  We followed the
same experimental protocol as the previous section, i.e., held out a
portion of the training set for validation to select hyperparameters.

Table~\ref{tab:covertyperes} summarizes the results.\footnote{When
comparing with other published results, be aware that many
authors adjust the task to be a binary classification task.}
GEM and deep GEM are exactly the same as in the previous
section, i.e., Algorithm~\ref{alg:gem} without and with
self-composition respectively.  RFF stands for Random Fourier
Features~\cite{rahimi2007random}, in which the Gaussian kernel
is approximated in the primal by a randomized cosine map;
we used logistic regression for the primal learning algorithm.
We treated the bandwidth and number of cosines as hyperparameters
to be optimized.

The relatively poor classification performance of RFF on Covertype
has been noted before~\cite{rahimi2007random}, a result we reproduce
here.  Instead of using the randomized feature map directly, however,
we can apply Algorithm~\ref{alg:gem} to the representation induced
by RFF, which we denote GEM + RFF.  This improves the classification
error with only modest increase in computation cost, e.g., in MATLAB
it takes 8 seconds to compute the randomized Fourier features, 58
seconds to (sequentially) solve the generalized eigenvalue problems
and compute the GEM feature representation, and 372 seconds to
optimize the logistic regression. The final error rate of 8.4\%
is a new record for this task. 

\begin{table}
\begin{center}
\begin{tabular}{c|c}
Method & Test Error Rate \\ \hline
GEM & 12.9\%  \\
RFF & 12.7\% \\
deep GEM & 9.8\% \\
GEM + RFF & 8.4\% \\
RBF kernel (exact) & 8.8\% 
\end{tabular}
\caption{Test error rates on Covertype.  The RBF kernel result is from~\cite{manik} where
    they also use a 90\%-10\% (but different) train-test split.}
\label{tab:covertyperes}
\end{center}
\vspace{-24pt}
\end{table}

\subsection{TIMIT}

TIMIT is a corpus of phonemically and lexically annotated speech
of English speakers of multiple genders and dialects~\cite{timit}.
Although the ultimate problem is sequence annotation, there is
a derived multiclass classification problem of predicting the
phonemic annotation associated with a short segment of audio.  Such a
classifier can be composed with standard sequence modeling techniques
to produce an overall solution, which has made the multiclass problem
a subject of research~\cite{hinton2012improving,hutchinson2012deep}.
In this experiment we focus exclusively on the multiclass problem.

We use a standard preprocessing of TIMIT as our initial
representation~\cite{hutchinson2012deep}.  Specifically the speech
is converted into feature vectors via the first to twelfth Mel
frequency cepstral coefficients and energy plus first and second
temporal derivatives.  This results in 39 coefficients per frame,
which is concatenated with 5 preceding and 5 following frames to
produce a 429 coefficient input to the classifier.  The targets
for the classifier are the 183 phone states (i.e., 61 phones each
in 3 possible states).

We use the standard training, development, and test sets of TIMIT.
As in previous experiments herein, hyperparameters are optimized
on the development set (using cross-entropy as the objective), but
unlike previous experiments we do not retrain with the development
set once hyperparameters are determined, in correspondence with the
experimental protocol used with the T-DSN~\cite{hutchinson2012deep}.

With 183 classes the all-pairs approach for generalized eigenvector
extraction is  unwieldy, so we used a randomized procedure to select
from which class pairs to extract features, by randomly positioning
the class labels on a hypercube and extracting generalized
eigenvectors only for immediate hyperneighbors.  For $k$ classes
this results in $O (k \log k)$ generalized eigenvalue problems.
Although we did not attempt a thorough exploration of different
strategies for subproblem selection, the hypercube heuristic yielded
better results for a given feature budget than either uniform random
selection over all class pairs or stratified random selection over
class pairs ensuring equal numbers of denominator or numerator
classes.  The resulting performance for five different choices of
random hypercube is shown in the row of Table~\ref{tab:timitres}
denoted GEM. We show both multiclass error rate as well as cross
entropy, the objective we are actually optimizing.

The random subproblem selection creates an opportunity to ensemble,
and empirically the resulting classifiers are sufficiently
diverse that ensembling yields a substantial improvement.
In Table~\ref{tab:timitres}, denoted GEM~ensemble, we show the
performance of the ensemble prediction of the 5 classifiers using
the geometric mean prediction (this is the prediction that minimizes
its average KL-divergence to each element of the ensemble).
The result matches the classification error and improves upon the
cross-entropy loss of the best published T-DSN.  This is remarkable
considering the T-DSN is a deep architecture employing between
8 and 13 stacked layers of nonlinear transformations, whereas
the GEM procedure produces a shallow architecture with a single
nonlinear layer.

\begin{table}
\begin{center}
\begin{tabular}{c|c|c}
\multirow{2}{*}{Method} & Frame & Cross \\
& State Error (\%)  & Entropy \\ \hline
GEM & $41.87 \pm 0.073 $ & $1.637\pm 0.001$ \\
T-DSN & 40.9 & 2.02 \\
GEM (ensemble) & 40.86 & 1.581
\end{tabular}
\caption{Results on TIMIT test set.  T-DSN is the best result from~\cite{hutchinson2012deep}.}
\label{tab:timitres}
\end{center}
\vspace{-24pt}
\end{table}

\section{Discussion}\label{sec:discussion}

Given the simplicity and empirical success of our method, we were
surprised to find considerable work on methods that only extract
the first generalized eigenvector \cite{mika2003constructing} but
very little work on using the top $m$ generalized eigenvectors.
Our experience is that additional eigenvectors provide complementary
information. Empirically, their inclusion in the final classifier
far outweighs the necessary increase in sample complexity, especially
given typical modern data set sizes.  Thus we believe this technique
should be valuable in other domains.

Of course our method will not be able to extract anything useful
if all classes have the same second moment but different higher
order statistics.  While our limited experience here suggests second
moments are informative for natural datasets, there are potential
benefits in using higher order moments. For example, we could replace
our class-conditional second moment matrix with a second moment
matrix conditioned on other events, informed by higher order moments.

As the number of class labels increases, say $k \geq 1000$, our
brute force all-pairs approach, which scales as $O(k^2)$, becomes
increasingly difficult both computationally and statistically: we
need to solve $O(k^2)$ eigenvector problems (possibly in parallel)
and deal with $O(k^2)$ features in the ultimate classifier.
Taking a step back, the object of our attention is the tensor $\E[x
\otimes x \otimes y]$ and in this paper we only studied one way
of selecting pairs of slices from it. In particular, our slices
are tensor contractions with one of the standard basis vectors in
$\R^k$. Clearly, contracting the tensor with any vector $u$ in $\R^k$
is possible. This contraction leads to a $d\times d$ second moment
matrix which averages the examples of the different classes in the
way prescribed by $u$. Any sensible, data-dependent way of picking
a good set of vectors $u$ should be able to reduce the dependence
on $k^2$.

The same issues also arise with a continuous $y$: how to define and
estimate the pairs of matrices whose generalized eigenvectors should
be extracted is not immediately clear. Still, the case where $y$ is
multidimensional (vector regression) can be reduced to the case of
univariate $y$ using the same technique of contraction with a vector
$u$. Feature extraction from a continuous $y$ can be done by 
discretization (solely for the purpose of feature extraction), 
which is much easier in the univariate case than in the multivariate
case.

In domains where examples exhibit large variation, or when labeled
data is scarce, incorporating prior knowledge is extremely important.
For example, in image recognition, convolutions and local pooling
are popular ways to generate representations that are invariant to
localized distortions. Directly exploiting the spatial or temporal
structure of the input signal, as well as incorporating other kinds
of invariances in our framework, is a direction for future work.

High dimensional problems create both computational and
statistical challenges.  Computationally, when $d > 10^6$, 
the solution of generalized eigenvalue problems can only
be performed via specialized libraries such as ScaLAPACK,
or via randomized techniques, such as those outlined 
in~\cite{halko2011finding,saibaba2013randomized}.  
Statistically, the finite-sample second
moment estimates can be inaccurate when the number of dimensions
overwhelms the number of examples. The effect of this 
inaccuracy on the extracted eigenvectors needs 
further investigation. In particular, it might be
unimportant for datasets encountered in practice, e.g., 
if the true class-conditional second moment matrices have low effective rank~\cite{bunea2012}.

Finally, our approach is simple to implement and well suited to
the distributed setting. Although a distributed implementation
is out of the scope of this paper, we do note that aspects
of Algorithm~\ref{alg:gem} were motivated by the desire for
efficient distributed implementation.  The recent success of
non-convex learning systems has sparked renewed interest in
non-convex representation learning. However, generic distributed
non-convex optimization is extremely challenging.  Our approach
first decomposes the problem into tractable non-convex subproblems
and then subsequently composes with convex techniques.  Ultimately we
hope that judicious application of convenient non-convex objectives,
coupled with convex optimization techniques, will yield competitive
and scalable learning algorithms.

 



\section{Conclusion}

We have shown a method for creating discriminative features
via solving generalized eigenvalue problems, and demonstrated
empirical efficacy via multiple experiments.  The method has
multiple computational and statistical desiderata.  Computationally,
generalized eigenvalue extraction is a mature numerical primitive,
and the matrices which are decomposed can be estimated using
map-reduce techniques.  Statistically, the method is invariant to
invertible linear transformations, estimation of the eigenvectors is
robust when the number of examples exceeds the number of variables,
and estimation of the resulting classifier parameters is eased due
to the parsimony of the derived representation.

Due to this combination of empirical, computational, and statistical
properties, we believe the method introduced herein has utility
for a wide variety of machine learning problems.

\section*{Acknowledgments} 
We thank John Platt and Li Deng for helpful discussions and assistance with the TIMIT experiments.


\bibliography{gev}
\bibliographystyle{icml2014}

\end{document}